\documentclass[runningheads]{llncs}
\usepackage[T1]{fontenc}
\usepackage{amsmath}
\usepackage{amssymb}
\usepackage{mathtools}
\usepackage{multirow}
\usepackage{cite}
\usepackage{hyperref}
\usepackage{enumerate,enumitem,tikz,graphicx,mathrsfs,eucal,verbatim, bbm, derivative}
\usepackage{algorithm}
\usepackage{algpseudocode}
\usepackage{caption}
\usepackage{subcaption}
\usepackage{wrapfig}
\usepackage{tabularx}
\usepackage{tabu}
\usepackage{booktabs}

\usepackage[misc]{ifsym}

\newcommand\blfootnote[1]{%
  \begin{NoHyper}%
  \renewcommand\thefootnote{}\footnote{#1}%
  \addtocounter{footnote}{-1}%
  \end{NoHyper}%
}

\newcommand{\methodname}{EquIN}

\begin{document}

\title{Equivariant Representation Learning \\ in the Presence of Stabilizers}
%
%

\author{Luis Armando Pérez Rey (\Letter) *\inst{1,2,3} 
\and
Giovanni Luca Marchetti *\inst{4} \and \\
Danica Kragic\inst{4}\ \and Dmitri Jarnikov \inst{1,3} \and Mike Holenderski\inst{1}}
\authorrunning{L. A. P. Rey, G. L. Marchetti et al.}
%
\institute{Eindhoven University of Technology, Eindhoven, The Netherlands \and
Eindhoven Artificial Intelligence Systems Institute, Eindhoven, The Netherlands
\and
Prosus, Amsterdam, The Netherlands
\and KTH Royal Institute of Technology, Stockholm, Sweden}

\tocauthor{Luis~Armando~Pérez~Rey, Giovanni~Luca~Marchetti,
Danica~Kragic, Dmitri~Jarnikov, Mike~Holenderski}
\toctitle{Equivariant Representation Learning in the Presence of Stabilizers}
\maketitle              

\begin{abstract}
We introduce Equivariant Isomorphic Networks (\methodname) -- a method for learning representations that are equivariant with respect to general group actions over data. Differently from existing equivariant representation learners, \methodname \ is suitable for group actions that are not free, i.e., that stabilize data via nontrivial symmetries. \methodname \ is theoretically grounded in the orbit-stabilizer theorem from group theory. This guarantees that an ideal learner infers isomorphic representations while trained on equivariance alone and thus fully extracts the geometric structure of data. We provide an empirical investigation on image datasets with rotational symmetries and show that taking stabilizers into account improves the quality of the representations.\blfootnote{*Equal Contribution}

\end{abstract}

\section{Introduction}

Incorporating data symmetries into deep neural representations defines a fundamental challenge and has been addressed in several recent works \cite{Quessard2020LearningEnvironments, Higgins2022Symmetry-BasedIntelligence, Cohen2014, Tonnaer2022QuantifyingDisentanglement, AhujaPropertiesLearning}. The overall aim is to design representations that preserve symmetries and operate coherently with respect  to them -- a functional property known as \emph{equivariance}. This is because the preservation of symmetries leads to the extraction of geometric and semantic structures in data,  which can be exploited for data efficiency and generalization \cite{Bengio2013}. As an example, the problem of \emph{disentangling} semantic factors of variation in data has been rephrased in terms of equivariant representations \cite{Higgins2018, Caselles-Dupre2019}. As disentanglement is known to be unfeasible with no inductive biases or supervision \cite{locatello2019challenging}, symmetries of data arise as a geometric structure that can provide weak supervision and thus be leveraged in order to disentangle semantic factors.    

\begin{figure}[h!]
  \begin{center}
    \includegraphics[width=.5\linewidth]{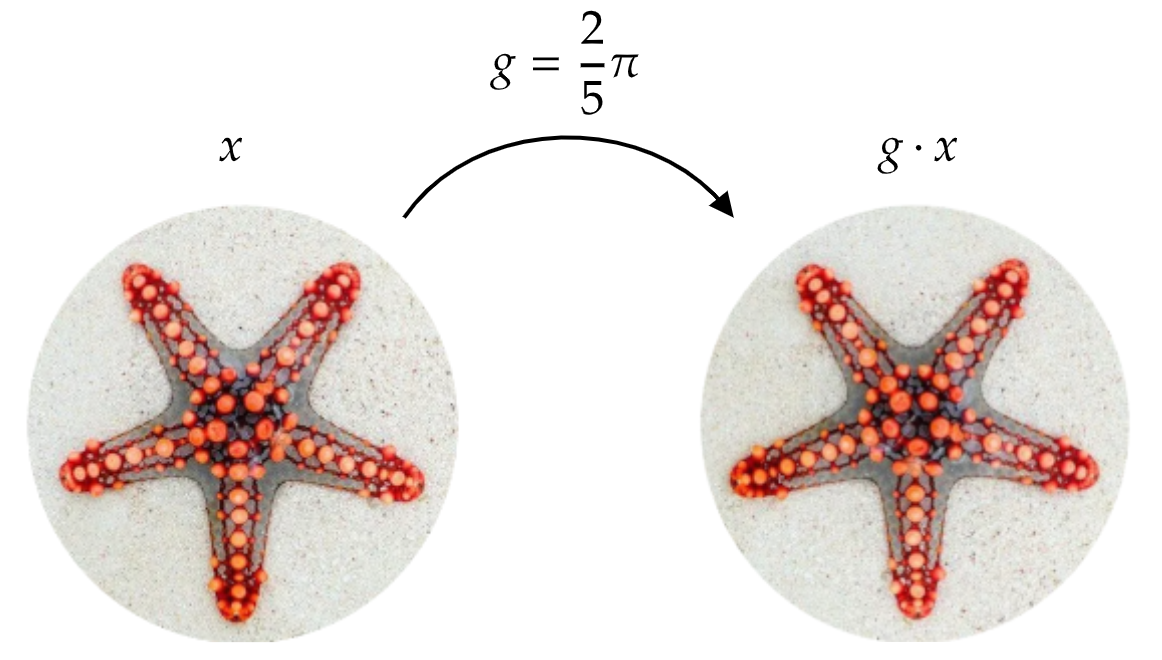}
  \end{center}
 \caption{An example of an action on data that is not free. The datapoint $x$ is stabilized by the symmetry $g \in G$.}\label{nonfreepic}
\end{figure}

The majority of models from the literature rely on the assumption that the group of symmetries acts \emph{freely} on data \cite{Marchetti2022EquivariantDecomposition} i.e., that no datapoint is stabilized by nontrivial symmetries. This avoids the need to model \emph{stabilizers} of datapoints, which are unknown subgroups of the given symmetry group. However, non-free group actions arise in several practical scenarios. This happens, for example, when considering images of objects acted upon by the rotation group via a change of orientation. Such objects may be symmetrical, resulting in rotations leaving the image almost identical and consequently ambiguous in its orientation, see Figure \ref{nonfreepic}. Discerning the correct orientations of an object is important for applications such as pose estimation \cite{Marchand2016} and reinforcement learning \cite{ha2018world}. This motivates the need to design equivariant representation learning frameworks that are capable of modeling stabilizers and therefore suit non-free group actions.

In this work, we propose a method for learning equivariant representation for general and potentially non-free group actions. Based on the Orbit-Stabilizer Theorem from group theory, we design a model that outputs subsets of the group, which represent the stabilizer subgroup up to a symmetry -- a group theoretical construction known as \emph{coset}. The representation learner optimizes an equivariance loss relying on supervision from symmetries alone. This means that we train our model on a dataset consisting of relative symmetries between pairs of datapoints, avoiding the need to know the whole group action over data a priori. From a theoretical perspective, the above-mentioned results from group theory guarantee that an ideal learner infers representations that are isomorphic to the original dataset. This implies that our representations completely preserve the symmetry structure while preventing any loss of information. We name our framework Equivariant Isomorphic Networks -- \methodname \ for short. In summary, our contributions include:

\begin{itemize}
\item A novel equivariant representation learning framework suitable for non-free group actions. 
\item A discussion grounded on group theory with theoretical guarantees for isomorphism representations.
\item An empirical investigation with comparisons to competing equivariant representation learners on image datasets.  
\end{itemize}

We provide Python code implementing our framework together with all the experiments at the following repository: \href{https://github.com/luis-armando-perez-rey/non-free}{\nolinkurl{luis-armando-perez-rey/non-free}}.

\section{Related Work}

In this section, we first briefly survey representation learning methods from the literature leveraging on equivariance. We then draw connections between equivariant representations and world models from reinforcement learning and discuss the role of equivariance in terms of disentangling semantic factors of data. 

{\setlength{\parindent}{0cm}

\textbf{Equivariant Representation Learning}. Several works in the literature have proposed and studied representation learning models that are equivariant with respect to a group of data symmetries. These models are typically trained via a loss encouraging equivariance on a dataset of relative symmetries between datapoints. What distinguishes the models is the choice of the latent space and of the group action over the latter. Euclidean latent spaces with linear or affine actions have been explored in \cite{guo2019affine, worrall2017interpretable, Quessard2020LearningEnvironments}. However, the intrinsic data manifold is non-Euclidean in general, leading to representations that are non-isomorphic and that do not preserve the geometric structure of the data. To amend this, a number of works have proposed to design latent spaces that are isomorphic to disjoint copies of the symmetry group \cite{hinton2011transforming, homeomorphic, Tonnaer2022QuantifyingDisentanglement, Marchetti2022EquivariantDecomposition}. When the group action is free, this leads to isomorphic representations and thus completely recovers the geometric structure of the data \cite{Marchetti2022EquivariantDecomposition}. However, the proposed latent spaces are unsuitable for non-free actions. Since they do not admit stabilizers, no equivariant map exists, and the model is thus unable to learn a suitable representation. In the present work, we extend this line of research by designing a latent space that enables learning equivariant representations in the presence of stabilizers. Our model implicitly represents stabilizer subgroups and leads to isomorphic representations for arbitrary group actions. 

\textbf{Latent World Models}. Analogously to group actions, Markov Decision Processes (MDPs) from reinforcement learning and control theory involve a, possibly stochastic, interaction with an environment. This draws connections between MDPs and symmetries since the latter can be thought of as transformations and, thus, as a form of interaction. The core difference is that in an MDP, no algebraic structure, such as a group composition, is assumed on the set of interactions. In the context of MDPs, a representation that is equivariant with respect to the agent's actions is referred to as latent \emph{World Model} \cite{ha2018world, kipf2019contrastive, park2022learning} or \emph{Markov Decision Process Homomorphism} (MDPH) \cite{van2020plannable}. In an MDPH the latent action is learned together with the representation by an additional model operating on the latent space. Although this makes MDPHs more general than group-equivariant models, the resulting representation is unstructured and uninterpretable. The additional assumptions of equivariant representations translate instead into the preservation of the geometric structure of data.  

\textbf{Disentanglement}. As outlined in \cite{Bengio2013}, a desirable property for representations is disentanglement, i.e., the ability to decompose in the representations the semantic factors of variations that explain the data. Although a number of methods have been proposed for this purpose \cite{higgins2017beta, chen2018isolating}, it has been shown that disentanglement is mathematically unachievable in an unbiased and unsupervised way \cite{locatello2019challenging}. As an alternative, the notion has been rephrased in terms of symmetry and equivariance \cite{Higgins2018}. It follows that isomorphic equivariant representations are guaranteed to be disentangled in this sense \cite{Tonnaer2022QuantifyingDisentanglement, Marchetti2022EquivariantDecomposition}. Since we aim for general equivariant representations that are isomorphic, our proposed method achieves disentanglement as a by-product. 
}

\section{Group Theory Background}
We review the fundamental group theory concepts necessary to formalize our representation learning framework. For a complete treatment, we refer to \cite{rotman2012introduction}. 
\begin{definition}\label{groupdef}
A group is a set $G$ equipped with a \emph{composition map} $G \times G \rightarrow G$ denoted by $(g,h) \mapsto gh$, an \emph{inversion map} $G \rightarrow G$ denoted by $g \mapsto g^{-1}$, and a distinguished \emph{identity element} $1 \in G$ such that for all $g, h, k \in G$:

\begin{center}
\begin{tabular}{ccc}
\emph{Associativity} &   \emph{Inversion} & \emph{Identity}  \\
 $g(hk) = (gh)k$ & \hspace{.2cm} $g^{-1}g = g g^{-1} = 1$  \hspace{.2cm} & $g1 = 1g = g$ 
\end{tabular}
\end{center}
\end{definition}

Elements of a group represent abstract symmetries. Spaces with a group of symmetries $G$ are said to be acted upon by $G$ in the following sense. 

\begin{definition}\label{actiondef}
An action by a group $G$ on a set $\mathcal{X}$ is a map $G \times \mathcal{X} \rightarrow \mathcal{X}$ denoted by $(g,x) \mapsto g \cdot x$, satisfying for all $g,h \in G, \ x \in \mathcal{X}$: 

\begin{center}
\begin{tabular}{ccc}
\emph{Associativity} & \hspace{1cm} & \emph{Identity} \\
$g\cdot (h \cdot x) = (gh) \cdot x$ & \hspace{1cm} &  $1 \cdot x = x$
 \end{tabular}
  \end{center}
\end{definition}

Suppose that $G$ acts on a set $\mathcal{X}$. The action defines a set of \emph{orbits} $\mathcal{X} / G$ given by the equivalence classes of the relation $x \sim y$ iff $y = g \cdot x$ for some $g \in G$. For each $x \in \mathcal{X}$, the \emph{stabilizer} subgroup is defined as 
\begin{equation}
 G_x = \{ g \in G \ | \ g \cdot x = x\}.
 \end{equation}
 Stabilizers of elements in the same orbit are conjugate, meaning that for each $x,y$ belonging to the same orbit $O$ there exists $h \in G$ such that $G_y = hG_x h^{-1}$. By abuse of notation, we refer to the conjugacy class $G_O$ of stabilizers for $O \in \mathcal{X} / G$. The action is said to be \emph{free} if all the stabilizers are trivial, i.e., $G_O = \{1 \}$ for every $O$.

We now recall the central notion for our representation learning framework. 
\begin{definition}
A map $\varphi: \ \mathcal{X} \rightarrow \mathcal{Z}$ between sets acted upon by $G$ is \emph{equivariant} if $\varphi(g\cdot x) = g \cdot \varphi(x)$ for every $x \in \mathcal{X}$ and $g \in G$. An equivariant bijection is referred to as \emph{isomorphism}.
\end{definition}
Intuitively, an equivariant map between $\mathcal{X}$ and $\mathcal{Z}$ preserves their corresponding symmetries. The following is the fundamental result on group actions \cite{rotman2012introduction}.
\begin{theorem}[Orbit-Stabilizer]\label{thm:orbstab}
The following holds:
\begin{itemize}
\item Each orbit $O$ is isomorphic to the set of (left) \emph{cosets} $G / G_O = \{ gG_O \ | \  g \in G \}$. In other words, there is an isomorphism: 
\begin{equation}\label{orbitstab}
 \mathcal{X} \simeq \coprod_{O \in \mathcal{X} / G } G / G_O \quad \subseteq 2^G \times \mathcal{X} / G
\end{equation}
where $2^G$ denotes the power-set of $G$ on which $G$ acts by left multiplication i.e., $g \cdot A = \{g a \ | \ a \in A \}$. 

\item Any equivariant map 
\begin{equation}
\varphi: \ \mathcal{X} \rightarrow \coprod_{O \in \mathcal{X} / G } G / G_O
\end{equation}
that induces a bijection on orbits is an isomorphism. 
\end{itemize}
\end{theorem}

Theorem \ref{thm:orbstab} describes arbitrary group actions completely and asserts that orbit-preserving equivariant maps are isomorphisms. Our central idea is to leverage on this in order to design a representation learner that is guaranteed to be isomorphic when trained on equivariance alone. 

\begin{figure}[h!]
  \begin{center}
    \includegraphics[width=.5\linewidth]{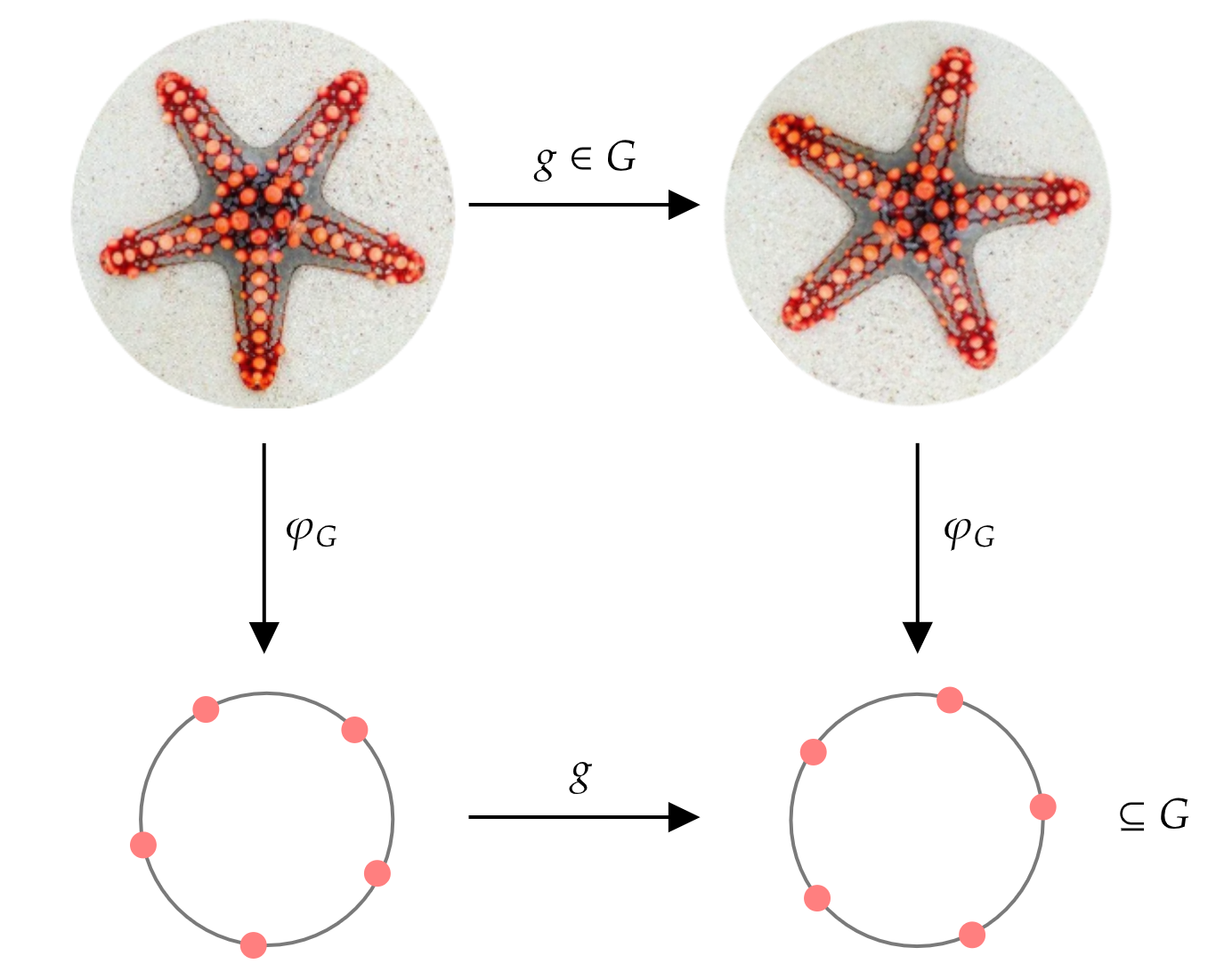}
  \end{center}
 \caption{An illustration of \methodname \ encoding data equivariantly as subsets of the symmetry group $G$. This results in representations that are suitable even when the action by $G$ on data is not free.  }
\end{figure}

\section{Equivariant Isomorphic Networks (\methodname)}
Our goal is to design an equivariant representation learner based on Theorem \ref{thm:orbstab}. We aim to train a model 
\begin{equation}
\varphi: \ \mathcal{X} \rightarrow \mathcal{Z}
\end{equation}
with a latent space $\mathcal{Z}$ on a loss encouraging equivariance.  The ideal choice for $\mathcal{Z}$ is given by $\coprod_{O \in \mathcal{X} / G } G / G_O$ since the latter is isomorphic to $\mathcal{X}$ (Theorem \ref{thm:orbstab}). In other words, $\varphi$ ideally outputs cosets of stabilizers of the input datapoints. However, while we assume that $G$ is known a priori, its action on $\mathcal{X}$ is not and has to be inferred from data. Since the stabilizers depend on the group action, they are unknown a priori as well. In order to circumvent the modeling of stabilizers and their cosets, we rely on the following simple result: 
\begin{proposition}\label{elementarprop}
Let $\varphi: \ \mathcal{X} \rightarrow 2^G$ be an equivariant map. Then for each $x\in \mathcal{X}$ belonging to an orbit $O$, $\varphi(x)$ contains a coset of (a conjugate of) $G_O$. 
\end{proposition}

\begin{proof}
Pick $x \in \mathcal{X}$. Then for every $g \in G_x$ it holds that $\varphi(x) = \varphi(g \cdot x) = g \cdot \varphi(x)$. In other words, $G_x h = h h^{-1} G_x h\subseteq  \varphi(x)$ for each $h \in \varphi(x)$. Since $h^{-1}G_xh$ is conjugate to $G_x$ the thesis follows.  
\end{proof}

Proposition \ref{elementarprop} enables $\varphi$ to output arbitrary subsets of $G$ instead of cosets of stabilizers. As long as those subsets are \emph{minimal} w.r.t. to inclusion, they will coincide with the desired cosets. \\

Based on this, we define the latent space of \methodname \ as $
\mathcal{Z} = \mathcal{Z}_G \times \mathcal{Z}_O$ and implement the map $\varphi$ as a pair of neural networks $\varphi_G : \mathcal{X} \rightarrow \mathcal{Z}_G$ and $\varphi_O : \mathcal{X} \rightarrow \mathcal{Z}_O$. The component $\mathcal{Z}_G$ represents cosets of stabilizers while $\mathcal{Z}_O$ represents orbits. Since the output space of a neural network is finite-dimensional, we assume that the stabilizers of the action are finite. The model $\varphi_G$ then outputs $N$ elements 
\begin{equation}\label{outphi}
\varphi_G(x) = \{ \varphi_G^1(x), \cdots, \varphi_G^N(x) \} \subseteq G 
\end{equation}
where $\varphi_G^i(x) \in G$ for all $i$. The hyperparameter $N$ should be ideally chosen larger than the cardinality of the stabilizers. On the other hand, the output of $\varphi_O$ consists of a vector of arbitrary dimensionality. The only requirement is that the output space of $\varphi_O$ should have enough capacity to contain the space of orbits $\mathcal{X} / G$. 

\subsection{Parametrizing $G$ via the Exponential Map}
The output space of usual machine learning models such as deep neural networks is Euclidean. However, $\varphi_G$ needs to output elements of the group $G$ (see Equation \ref{outphi}), which may be non-Euclidean as in the case of $G=\textnormal{SO}(n)$. Therefore, in order to implement $\varphi_G$, it is necessary to parametrize $G$. To this end, we assume that $G$ is a differentiable manifold, with differentiable composition and inversion maps, i.e., that $G$ is a \emph{Lie group}. One can then define the \emph{Lie algebra} $\mathfrak{g}$ of $G$ as the tangent space to $G$ at $1$.

We propose to rely on the \emph{exponential map} $\mathfrak{g} \rightarrow G$, denoted by $v \mapsto e^v$, to parametrize $G$. This means that $\varphi_G$ first outputs $N$ elements $ \varphi_G(x) = \{ v^1, \cdots, v^N \} \subseteq \mathfrak{g}$ that get subsequently mapped into $G$ as $\{e^{v^1}, \cdots, e^{v^N} \}$. Although the exponential map can be defined for general Lie groups by solving an appropriate ordinary differential equation, we focus on the case $G \subseteq \textnormal{GL}(n)$. The Lie algebra $\mathfrak{g}$ is then contained in the space of $n \times n$ matrices and the exponential map amounts to the matrix Taylor expansion 
\begin{equation}
e^v = \sum_{k \geq 0} \frac{v^k}{k!}
\end{equation}
where $v^k$ denotes the power of $v$ as a matrix. For specific groups, the latter can be simplified via simple closed formulas. For example, the exponential map of $\mathbb{R}^n$ is the identity while for $\textnormal{SO}(3)$ it can be efficiently computed via the Rodrigues' formula \cite{liang2018efficient}. 

\subsection{Training Objective}\label{sec:train}
As mentioned, our dataset $\mathcal{D}$ consists of samples from the unknown group action. This means that datapoints are triplets $(x,g,y) \in \mathcal{X} \times G \times \mathcal{X}$ with $y = g \cdot x$. Given a datapoint $(x,g,y) \in \mathcal{D}$ the learner $\varphi_G$ optimizes the equivariance loss over its parameters: 
\begin{equation}\label{eq:equivariance}
 \mathcal{L}_G(x,g,y) = d(g\cdot 
 \varphi_G(x), \ \varphi_G(y))
\end{equation}
where $d$ is a semi-metric for sets. We opt for the asymmetric \emph{Chamfer distance} 
\begin{equation}
d(A,B) = \frac{1}{|A|} \sum_{a \in A } \min_{b \in B} d_G(a,b)
\end{equation}
because of its differentiability properties. Any other differentiable distance between sets of points can be deployed as an alternative. Here $d_G$ is a metric on $G$ and is typically set as the squared Euclidean for $G=\mathbb{R}^n$ and as the squared Frobenius for $G = \textnormal{SO}(n)$. As previously discussed, we wish $\varphi_G(x)$, when seen as a set, to be minimal in cardinality. To this end, we add the following regularization term measuring the discrete entropy: 
\begin{equation}\label{eq:eqloss}
\widetilde{\mathcal{L}}_G(x) = \frac{\lambda}{N^2} \sum_{1 \leq i,j \leq N} d_G(\varphi^i_G(x), \ \varphi^j_G(x)) 
\end{equation}
where $\lambda$ is a weighting hyperparameter. On the other hand, since orbits are invariant to the group action $\varphi_O$ optimizes a \emph{contrastive loss}. We opt for the popular InfoNCE loss from the literature \cite{chen2020simple}:
\begin{equation}\label{eq:infonce}
    \mathcal{L}_O(x,y)  = d_O(\varphi_O(x), \ \varphi_O(y)) + \log \mathbb{E}_{x'}\left[e^{-d_O(\varphi_O(x'), \ \varphi_O(x))}\right]
\end{equation}
where $x'$ is marginalized from $\mathcal{D}$. As customary for the InfoNCE loss, we normalize the output of $\varphi_O$ and set $d_O(a,b) = - \cos(\angle{ab} ) = - a \cdot b$. The second summand of $\mathcal{L}_O$ encourages injectivity of $\varphi_O$ and as such prevents orbits from overlapping in the representation. 

The Orbit-Stabilizer Theorem (Theorem \ref{thm:orbstab}) guarantees that if \methodname \ is implemented with ideal learners $\varphi_G, \varphi_O$ then it infers isomorphic representations in the following sense. If the $\mathcal{L}_G(x,g,y)$ and the first summand of $\mathcal{L}_O(x,y)$ vanish for every $(x,g,y)$ then $\varphi$ is equivariant. If moreover the regularizations, $\widetilde{\mathcal{L}}_G$ and the second summand of $\mathcal{L}_O$, are at a minimum then $\varphi_G(x)$ coincides with a coset of $G_O$ for every $x \in O$ (Proposition \ref{elementarprop}) and $\varphi_O$ is injective. The second claim of Theorem \ref{thm:orbstab} implies then that the representation is isomorphic on its image, as desired. 

\section{Experiments}
We empirically investigate \methodname \ on image data acted upon by a variety Lie groups. Our aim is to show both qualitatively and quantitatively that \methodname \ reliably infers isomorphic equivariant representations for non-free group actions. 

We implement the neural networks $\varphi_G$ and $\varphi_O$ as a ResNet18 \cite{Sangeetha2006}. For a datapoint $x\in \mathcal{X}$, the network implements multiple heads to produce embeddings $\{ \varphi^1_G(x), \cdots , \varphi_G^N(x)\} \subseteq G$. The output dimension of $\varphi_O$ is set to $3$. We train the model for $50$ epochs using the AdamW optimizer \cite{Loshchilov2019DecoupledRegularization} with a learning rate of $10^{-4}$ and batches of 16 triplets $(x,g, y)\in \mathcal{D}$. 

\subsection{Datasets}\label{sec:datasets}
 We consider the following datasets consisting of $64 \times 64$ images subject to non-free group actions. Samples from these datasets are shown in Figure~\ref{fig:toruses}.
\begin{itemize}
\item \textsc{Rotating Arrows:} images of radial configurations of $\nu \in \{1,2,3, 4,5\}$ arrows rotated by $G = \textnormal{SO}(2)$. The number of arrows $\nu$ determines the orbit and the corresponding stabilizer is (isomorphic to) the cyclic group $C_\nu$ of cardinality $\nu$. The dataset contains 2500 triplets $(x,g,y)$ per orbit.

 \item \textsc{Colored Arrows:} images similar to \textsc{Rotating Arrows} but with the arrows of five different colors. This extra factor produces additional orbits with the same stabilizer subgroups. The number of orbits is therefore $25$. The dataset contains 2000 triplets per orbit. 
 
 \item \textsc{Double Arrows:} images of two radial configurations of $2, 3$ and $3,5$ arrows respectively rotated by the torus $G = \mathrm{SO}(2)\times \mathrm{SO}(2)$. The action produces two orbits with stabilizers given by products of cyclic groups: $C_2 \times C_3$ and $C_3 \times C_5$ respectively. The dataset contains 2000 triplets per orbit.

 \item \textsc{ModelNet:} images of monochromatic objects from ModelNet40 \cite{Wu_2015_CVPR} rotated by $G = \mathrm{SO}(2)$ along an axis. We consider five objects: an airplane, a chair, a lamp, a bathtub and a stool. Each object corresponds to a single orbit. The lamp, the stool and the chair have the cyclic group $C_4$ as stabilizer while the action over the airplane and the bathub is free. The dataset contains 2500 triplets per orbit.
 
 \item \textsc{Solids:} images of a monochromatic tetrahedron, cube and icosahedron \cite{implicitpdf2021} rotated by $G = \mathrm{SO}(3)$. Each solid defines an orbit, and the stabilizers of the tetrahedron, the cube, and the icosahedron are subgroups of order $12$, $24$ and $60$ respectively. The dataset contains 7500 triplets per orbit.
 \end{itemize}

\subsection{Comparisons}\label{sec:baselines}

 We compare \methodname \ with the following two equivariant representation learning models. 
\begin{itemize}
\item \textit{Baseline:} a model corresponding to \methodname \ with $N=1$ where $\varphi_G$ outputs a single element of $G$. The latent space is $\mathcal{Z} = G \times  \mathcal{Z}_O$, on which $G$ acts freely. We deploy this as the baseline since it has been proposed with minor variations in a number of previous works \cite{Caselles-Dupre2019,Painter2020, Marchetti2022EquivariantDecomposition, Tonnaer2022QuantifyingDisentanglement} assuming free group actions. 

\item \textit{Equivariant Neural Renderer (ENR):} a model from \cite{renderer} implementing a tensorial latent space $\mathcal{Z} = \mathbb{R}^{S^3}$, thought as a scalar signal space on a $S\times S \times S$ grid in $\mathbb{R}^3$. The group $\textnormal{SO}(3)$ act \emph{approximately} on $\mathcal{Z}$ by rotating the grid and interpolating the obtained values. The model is trained jointly with a decoder $\psi: \ \mathcal{Z} \rightarrow \mathcal{X}$ and optimizes a variation of the equivariance loss that incorporates reconstruction: $\mathbb{E}_{x,g, y=g\cdot x}[d_\mathcal{X}(y, \ \psi(g \cdot \varphi(x))   )]$ where $d_\mathcal{X}$ is the binary cross-entropy for normalized images. Although the action on $\mathcal{Z}$ is free, the latent discretization and consequent interpolation make the model only approximately equivariant.  
Similarly to \methodname , we implement ENR as ResNet18. As suggested in the original work \cite{renderer} we deploy $3$D convolutional layers around the latent and set to zero the latent dimensions outside a ball. We set $S = 8$ with 160 non-zero latent dimensions since this value is comparable to the latent dimensionality of \methodname, between 7 and 250 dimensions depending on $N$, making the comparison fair. Note that ENR is inapplicable to \textsc{Double Arrows} since its symmetry group is not naturally embedded into $\textnormal{SO}(3)$. 
\end{itemize}


 \subsection{Quantitative Results}

 In order to quantitatively compare the models, we rely on the following evaluation metrics computed on a test dataset $\mathcal{D}_{\textnormal{test}}$ consisting of $10 \%$ of the corresponding training data: 
  \begin{figure*}[b!]
\centering
 \includegraphics[width=.6\linewidth]{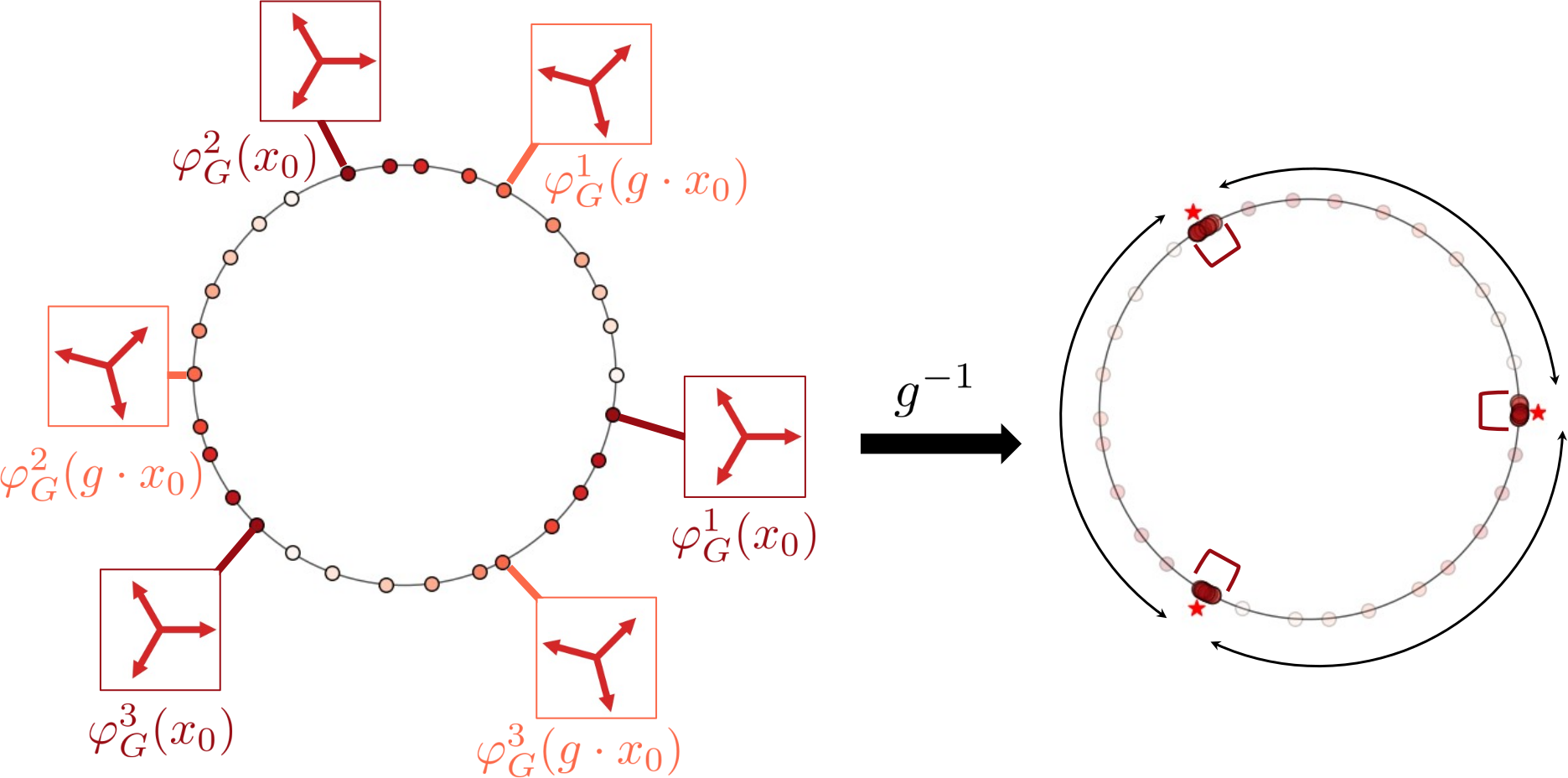}
\caption{Diagram explaining the estimation of the disentanglement metric for \methodname{}. This example assumes that $G = \mathrm{SO}(2)$ and that $A$ is the identity.}
\label{fig:metric}
\end{figure*}

\begin{itemize}
\item \textit{Hit-Rate:} a standard score comparing equivariant representations with different latent space geometries \cite{kipf2019contrastive}. Given a test triple $(x, g, y=g\cdot x) \in \mathcal{D}_{\textnormal{test}}$, we say that `$x$ hits $y$' if $\varphi(y)$ is the nearest neighbor in $\mathcal{Z}$ of $g \cdot \varphi(x)$ among a random batch of encodings $\{ \varphi(x) \}_{x \in \mathcal{B}}$ with $|\mathcal{B}|=20$. The hit-rate is then defined as the number of times $x$ hits $y$ divided by the test set size. For each model, the nearest neighbor is computed with respect to the same latent metric $d$ as the one used for training. Higher values of the metric are better.

 \item \textit{Disentanglement:} an evaluation metric proposed in \cite{Tonnaer2022QuantifyingDisentanglement} to measure disentanglement according to the symmetry-based definition of \cite{Higgins2018}. This metric is designed for groups in the form $G = \mathrm{SO}(2)^T$ and therefore is inapplicable to the \textsc{Solids} dataset. Per orbit, the test set is organized into datapoints of the form $y = g\cdot x_0$ where $x_0$ is an arbitrary point in the given orbit. In order to compute the metric, the test dataset is encoded into $\mathcal{Z}$ via the given representation and then projected to $\mathbb{R}^{2T}$ via principal component analysis. Then for each independent copy of $\textnormal{SO}(2) \subseteq G$, a group action  on the corresponding copy of $\mathbb{R}^2$ is inferred by fitting parameters via a grid search. Finally, the metric computes the average dispersion of the transformed embeddings as the variance of $g^{-1} \cdot A \varphi_G(y)$. For \methodname{}, we propose a modified version  accounting for the fact that $\varphi_G$ produces multiple points in $G$ using the Chamfer distance $d$ and averaging the dispersion with respect to each transformed embedding, see Figure~\ref{fig:metric}.  The formula for computing the metric is given by:

\begin{equation}
    \mathbb{E}_{y,y'}[d(h^{-1} \cdot A \varphi_G(y'), \ g^{-1} \cdot A \varphi_G(y))]
\end{equation}
where $y = g\cdot x_0$ and $y' = h \cdot x_0$. Lower values of the metric are better. 
\end{itemize}

\begin{table*}[h!]
  \centering
  \caption{Mean and standard deviation of the metrics across five repetitions. The number juxtaposed to the name of \methodname{} indicates the cardinality $N$ of the output of $\varphi_G$.
  \hfill \break} 

  \label{tab:quantitative-results}
  {\begin{tabular}{p{0.2\textwidth}p{0.2\textwidth}p{0.25\textwidth}p{0.2\textwidth}}
  \toprule
      Dataset & Model & Disentanglement $( \downarrow)$& Hit-Rate $(\uparrow)$ \\
  \hline

    \multirow{4}{*}{\shortstack{\textsc{Rotating}\\ \textsc{Arrows}}}& Baseline & $1.582_{\pm 0.013}$ & $0.368_{\pm0.004}$\\

 & \methodname5 & $\boldsymbol{0.009}_{\pm 0.005}$& $0.880_{\pm0.021}$\\

 & \methodname10 & $0.092_{\pm 0.063}$& $0.857_{\pm0.050}$\\
& ENR & $0.077_{\pm 0.028}$& $\boldsymbol{0.918}_{\pm 0.009}$\\


 \hline
 \multirow{4}{*}{\shortstack{\textsc{Colored}\\ \textsc{Arrows}}}& Baseline & $1.574_{\pm 0.007}$& $0.430_{\pm 0.004}$\\

 & \methodname5 & $0.021_{\pm 0.015}$ & $0.930_{\pm 0.055}$\\

 & \methodname10 & $\boldsymbol{0.001}_{\pm 0.001}$& $\boldsymbol{0.976}_{\pm 0.005}$\\


 & ENR & $0.106_{\pm 0.032}$& $0.949_{\pm 0.018}$\\

\hline

\multirow{4}{*}{\shortstack{\textsc{Double}\\ \textsc{Arrows}}}& Baseline & $1.926_{\pm0.019}$ & $0.023_{\pm0.004}$\\

& \methodname6 & $0.028_{\pm0.006}$& $0.512_{\pm0.011}$\\

 & \methodname15 & $0.004_{\pm0.001}$ & $0.820_{\pm0.104}$\\

 & \methodname20 & $\boldsymbol{0.002}_{\pm0.001}$& $\boldsymbol{0.934}_{\pm0.020}$ \\

 \hline

     \multirow{4}{*}{\textsc{ModelNet}}& Baseline & $1.003_{\pm 0.228}$& $0.538_{\pm0.086 }$\\
  & \methodname4 & $0.012_{\pm0.022 }$& $0.917_{\pm0.074 }$ \\

 & \methodname10 & $\boldsymbol{0.003}_{\pm 0.001}$& $\boldsymbol{0.910}_{\pm0.011 }$ \\

  & ENR & $0.037_{\pm0.038 }$ & $0.817_{\pm0.085 }$ \\

\hline

 \multirow{6}{*}{\textsc{Solids}}& Baseline & -& $0.123_{\pm 0.007}$\\
& \methodname12 & -& $0.126_{\pm 0.004}$\\
 & \methodname24 & -& $0.139_{\pm 0.056}$\\
 & \methodname60 & -& $0.596_{\pm 0.106}$\\
 & \methodname80 & -& $\boldsymbol{0.795}_{\pm 0.230}$\\
 & ENR & -& $0.772_{\pm 0.095}$\\

  \bottomrule
      \end{tabular}}
\end{table*}

The results are summarized in Table \ref{tab:quantitative-results}. \methodname \ achieves significantly better scores than the baseline. The latter is unable to model the stabilizers in its latent space, leading to representations of poor quality and loss of information. ENR is instead competitive with \methodname. Its latent space suits non-free group actions since stabilizers can be modelled as signals over the latent three-dimensional grid. ENR achieves similar values of hit-rate compared to \methodname. The latter generally outperforms ENR, especially on the \textsc{ModelNet} dataset, while is outpermformed on \textsc{Rotating Arrows}. According to the disentanglement metric, \methodname{} achieves significantly lower scores than ENR. This is probably due to the fact the latent group action in ENR is approximate, making the model unable to infer representations that are equivariant at a granular scale.


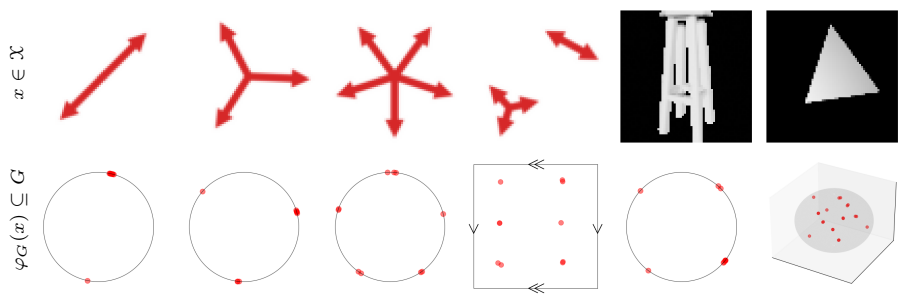
\begin{figure*}[h!]
\centering
 \includegraphics[width=.15\linewidth]{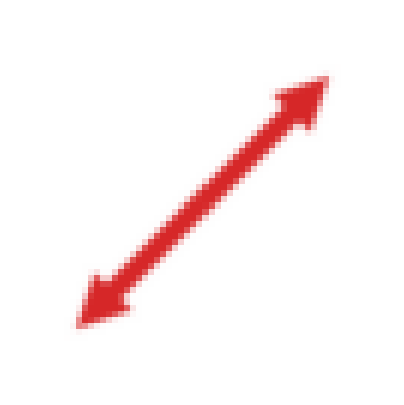}
 \includegraphics[width=.15\linewidth]{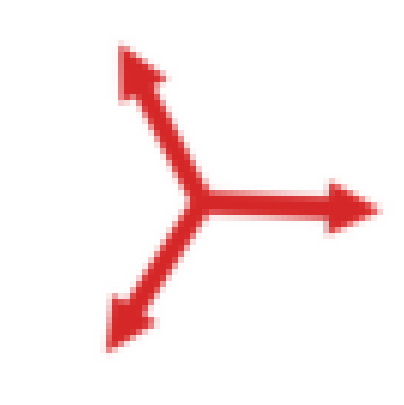}
 \includegraphics[width=.15\linewidth]{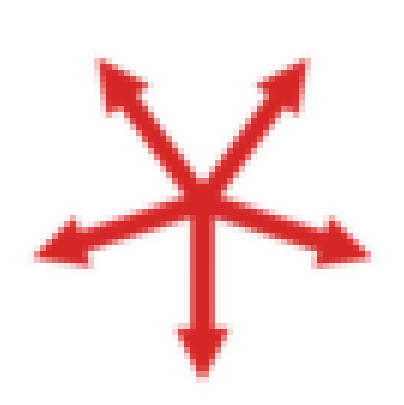}
 \includegraphics[width=.15\linewidth]{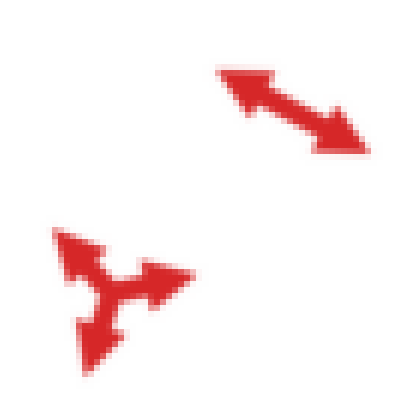}
  \includegraphics[width=.15\linewidth]{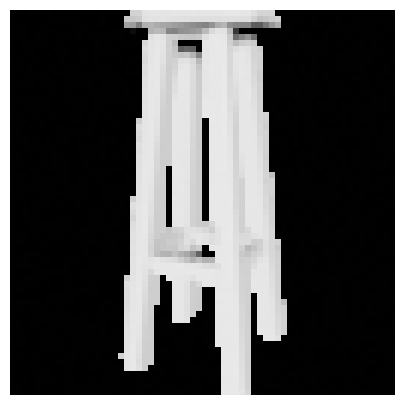}
  \includegraphics[width=.15\linewidth]{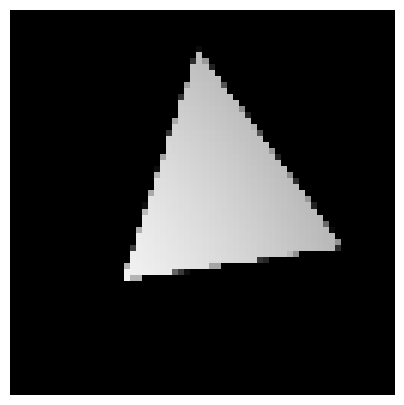}

 \includegraphics[width=.15\linewidth]{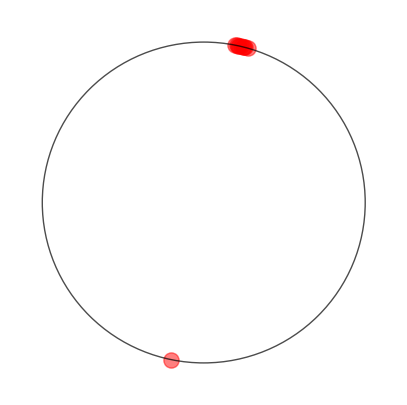}
 \includegraphics[width=.15\linewidth]{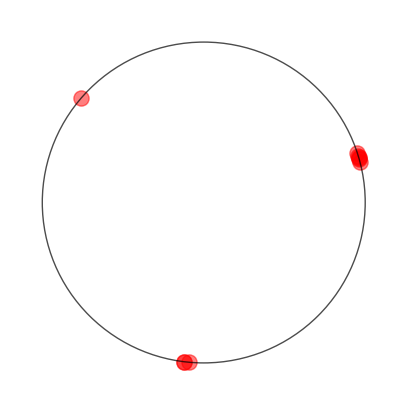}
 \includegraphics[width=.15\linewidth]{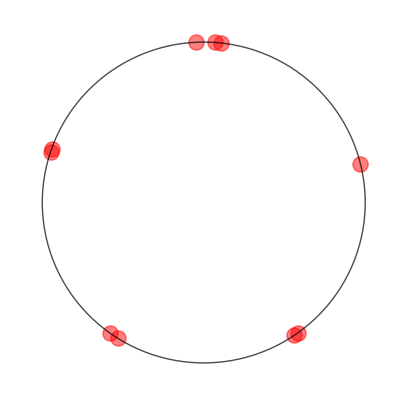}
 \includegraphics[width=.15\linewidth]{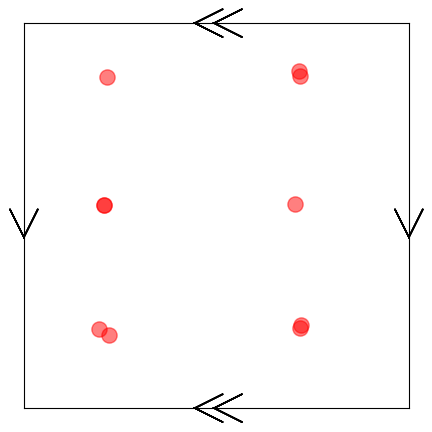}
 \includegraphics[width=.15\linewidth]{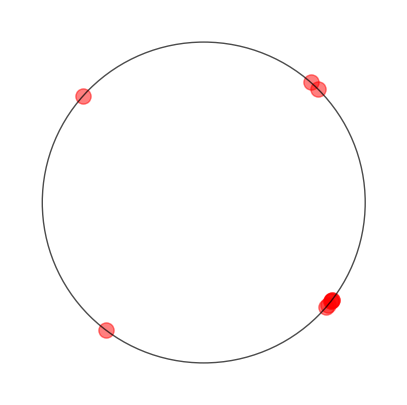}
 \includegraphics[width=.16\linewidth]{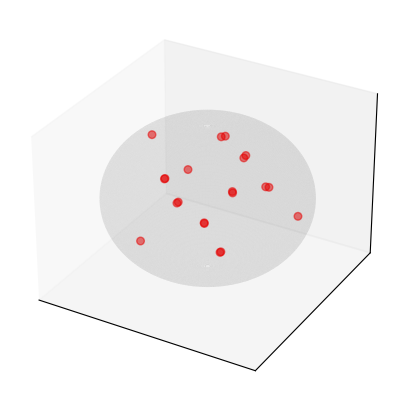}

\begin{picture}(0,0)
    \put(-173,85){\rotatebox{90}{\scriptsize $x \in \mathcal{X}$}}
    \put(-173,20){\rotatebox{90}{\scriptsize $\varphi_G(x) \subseteq G$}}
\end{picture}
\caption{Visualization of datapoints $x$ and the corresponding predicted (coset of the) stabilizer $\varphi_G(x)$. For \textsc{Double Arrows}, the torus $G = \textnormal{SO}(2) \times \textnormal{SO}(2)$ is visualized as an identified square. For the tetrahedron from \textsc{Solids}, $G$ is visualized as a projective space $\mathbb{RP}^3 \simeq \textnormal{SO}(3)$.}
\label{fig:toruses}
\end{figure*}

 \subsection{Qualitative Results}

\begin{figure*}[h!]
\centering
 \includegraphics[width=.35\linewidth]{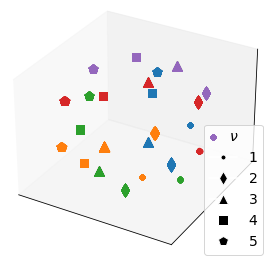} \hspace{6em}
 \includegraphics[width=.3\linewidth]{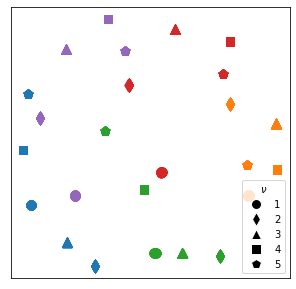}
\caption{Embeddings $\varphi_O(x) \in \mathcal{Z}_O \subseteq  \mathbb{R}^3$ for $x$ in \textsc{Colored Arrows}. Each symbol represents the ground-truth cardinality $\nu = |G_x|$ of the stabilizer while the color of the symbol represents the corresponding color of the arrow (left). The same embeddings are projected onto $\mathbb{R}^2$ via principal component analysis (right).}
\label{fig:orbits}
\end{figure*}

We provide a number of visualizations as a qualitative evaluation of \methodname. Figure \ref{fig:toruses} illustrates the output of $\varphi_G$ on the various datasets. As can be seen, \methodname \ correctly infers the stabilizers i.e., the cyclic subgroups of $\textnormal{SO}(2)$ and the subgroup of $\textnormal{SO}(3)$ of order $12$. When $N$ is larger than the ground-truth cardinalities of stabilizers, the points $\varphi_G^i$ are overlapped and collapse to the number of stabilizers as expected. Figure \ref{fig:orbits} displays the output of $\varphi_O$ for data from \textsc{Colored Arrows}. The orbits are correctly separated in $\mathcal{Z}_O$. Therefore, the model is able to distinguish data due to variability in the number $\nu$ of arrows as well as in their color.

\begin{figure*}[h!]
\centering
 \includegraphics[width=.43\linewidth]{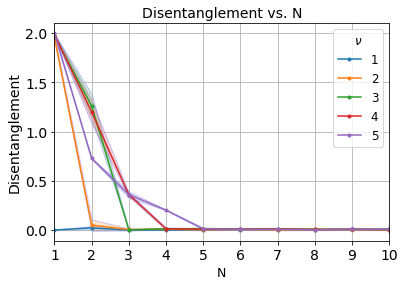}
 \includegraphics[width=.43\linewidth]{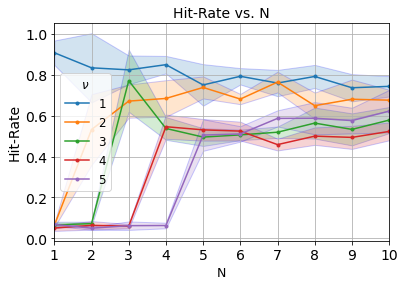}
 
\caption{Disentanglement and hit-rate for models trained with different values of $N$. Each line in the plot represents the results of a model trained on a dataset with a single orbit whose stabilizer has cardinality $\nu$. The plots show the mean and standard deviation across five repetitions.}
\label{fig:results_single_arrows}
\end{figure*}

\begin{figure*}[h!]
\centering
 \includegraphics[width=.43\linewidth]{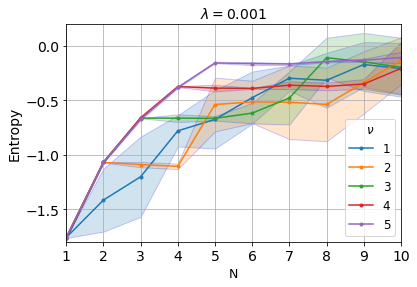}
 \includegraphics[width=.43\linewidth]{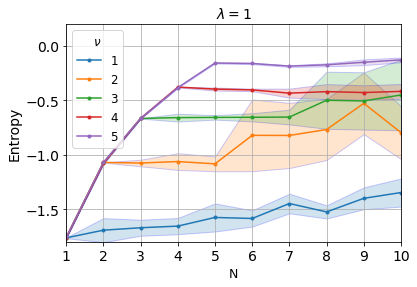}

 \includegraphics[width=.20\linewidth]{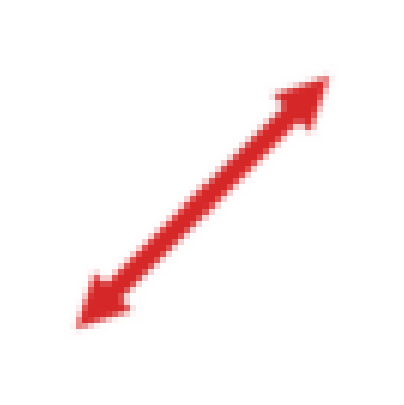}
   \hspace{2em}
 \includegraphics[width=.20\linewidth]{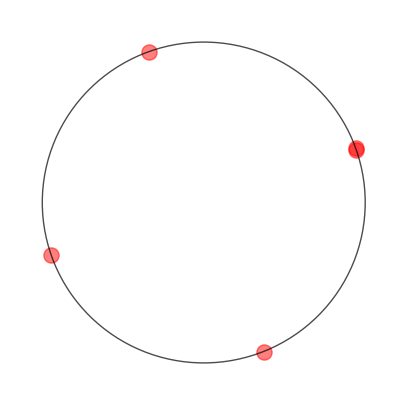}
 \includegraphics[width=.20\linewidth]{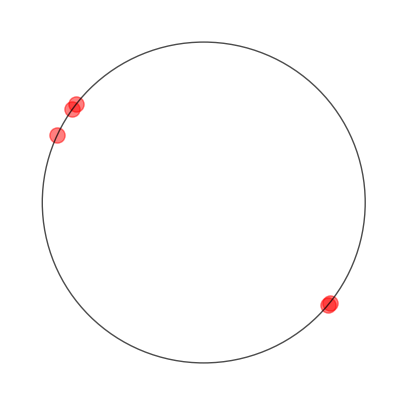}
  \includegraphics[width=.20\linewidth]{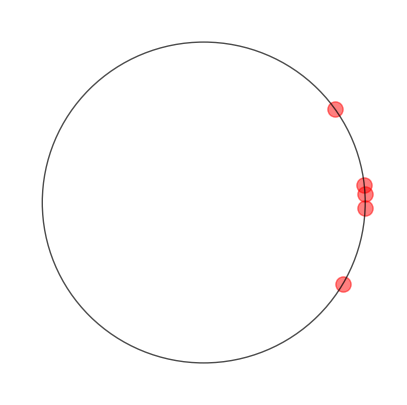}
 
\begin{picture}(0,0)
    \put(-47,5){\rotatebox{0}{\footnotesize $\lambda=0.001$}}
    \put(35,5){\rotatebox{0}{\footnotesize $\lambda=1$}}
    \put(107,5){\rotatebox{0}{\footnotesize $\lambda=10$}}
    \put(-115,5){\rotatebox{0}{\footnotesize $x$}}
    \put(-70,30){\rotatebox{90}{\footnotesize $\varphi_G(x)$}}
\end{picture}

\caption{Discrete entropy for models trained on the arrows dataset with different cardinalities of stabilizer $\nu$ and two distinct values of $\lambda$ (top row). Example embeddings $\varphi_G(x)$ obtained for a datapoint $x$ with two stabilizers obtained with models using $\lambda \in \{0.001, 1, 10\}$ (bottom row).}
\label{fig:results_equiv_lambda}
\end{figure*}
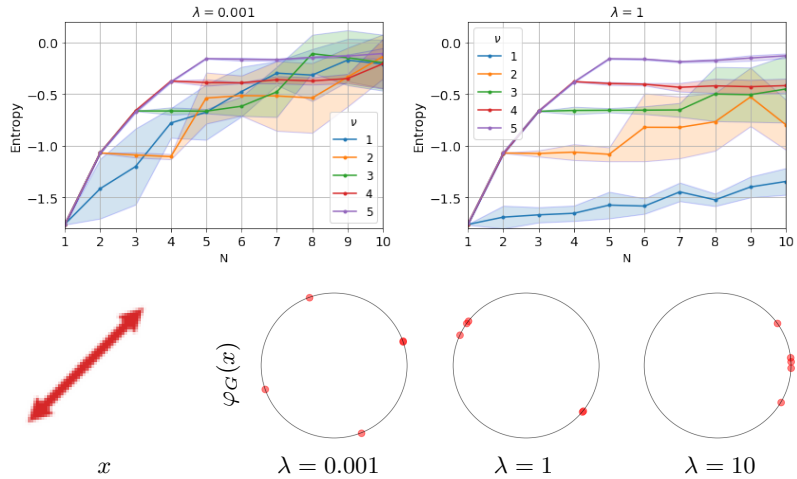

\begin{figure*}[h!]
\centering
 \includegraphics[width=.43\linewidth]{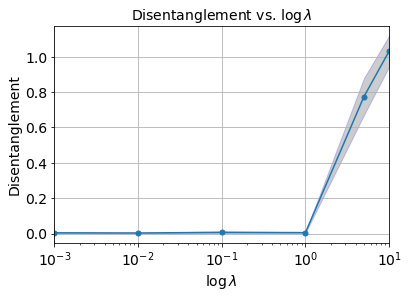}
\includegraphics[width=.43\linewidth]{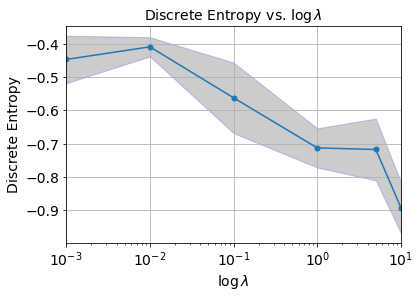}

 \includegraphics[width=.43\linewidth]{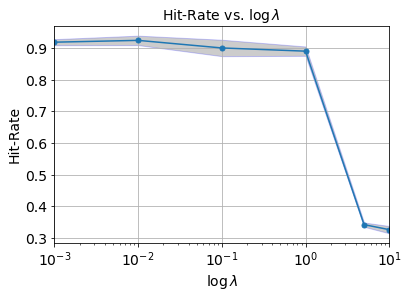}

\caption{Disentanglement, discrete entropy and hit-rate for models trained with different values of $\lambda$ and fixed $N=5$. The training dataset corresponds to the rotating arrows with $\nu \in \{1, 2, 3, 4, 5\}$. Each line shows the mean and standard deviation across five repetitions.}
\label{fig:results_lambdas}
\end{figure*}

\subsection{Hyperparameter Analysis}
For our last experiment, we investigate the effects of the hyperparameters $N$ and $\lambda$ when training \methodname{} on datasets with different numbers of stabilizers. 

First, we show that a value of $N$ larger than the cardinality of the stabilizers is necessary to achieve good values of disentanglement, and hit-rate for datasets with non-free group action, see Figure~\ref{fig:results_single_arrows}. However, large values of $N$ can result in non-collapsing embeddings $\varphi_G$ corresponding to non-minimal cosets of the stabilizers. In these cases, the regularization term of Equation \ref{eq:eqloss} and its corresponding weight $\lambda$ plays an important role. 

The bottom row of Figure~\ref{fig:results_equiv_lambda} shows the embeddings $\varphi_G(x)$ learnt for a datapoint $x\in \mathcal{X}$ with stabilizer $G_x \simeq C_2$ of cardinality two. The plots show how for low values of $\lambda$, the network converges to a non-minimal set. When an optimal value is chosen, such as $\lambda=1$, the embeddings obtained with $\varphi_G$ collapse to a set with the same cardinality as the stabilizers. If $\lambda$ is too large, the embeddings tend to degenerate and collapse to a single point. 

If the value of $\lambda$ is too small, the discrete entropy of the learnt embeddings is not restricted. It continues to increase even if the number of embeddings matches the correct number of stabilizers. When an appropriate value of $\lambda$ is chosen, the entropy becomes more stable as the embeddings have converged to the correct cardinality.

The plots in Figure~\ref{fig:results_lambdas} show the inverse relationship between $\lambda$ and the entropy of the encoder $\varphi_G$ that describes the collapse of the embeddings. The collapse of the embeddings also results in a lower performance of disentanglement and hit-rate by the models as seen for higher values of $\lambda>1$. Throughout the experiments, we fix the value of $\lambda = 1$ except for \textsc{Solids} where a value of $\lambda=10$ was chosen since the number $N$ used is larger.

\section{Conclusions and Future Work}
In this work, we introduced \methodname,  a method for learning equivariant representations for possibly non-free group actions. We discussed the theoretical foundations and empirically investigated the method on images with rotational symmetries. We showed that our model can capture the cosets of the group stabilizers and separate the information characterizing multiple orbits. 

\methodname \  relies on the assumption that the stabilizers of the group action are finite. However, non-discrete stabilizer subgroups sometimes occur in practice, e.g., in continuous symmetrical objects such as cones, cylinders or spheres. Therefore, an interesting future direction is designing an equivariant representation learner suitable for group actions with non-discrete stabilizers. 

\section*{Acknowledgements}
This work was supported by the Swedish Research Council, the Knut and Alice Wallenberg Foundation and the European Research Council (ERC-BIRD-884807). This work has also received funding from the NWO-TTW Programme “Efficient Deep Learning” (EDL) P16-25.

\section*{Ethical Statement}
The work presented in this paper consists of a theoretical and practical analysis on learning representations that capture the information about symmetry transformations observed in data. Due to the nature of this work as fundamental research, it is challenging to determine any direct adverse ethical implications that might arise. However, we think that any possible ethical implications of these ideas would be a consequence of the possible applications to augmented reality, object recognition, or reinforcement learning among others. The datasets used in this work consist of procedurally generated images with no personal or sensitive information.

\bibliographystyle{splncs04}
\bibliography{main}

\end{document}